\documentclass{siamart190516} 

\usepackage{hyperref}
\usepackage{url}
\usepackage{graphicx}

\newcommand{\abs}[1]{\left|#1\right|}

\newcommand{\ts}{\textsuperscript}

\usepackage{amsmath,amsfonts,bm}

\def\eqref#1{equation~\ref{#1}}

\def\floor#1{\lfloor #1 \rfloor}
\def\1{\bm{1}}

\DeclareMathAlphabet{\mathsfit}{\encodingdefault}{\sfdefault}{m}{sl}
\SetMathAlphabet{\mathsfit}{bold}{\encodingdefault}{\sfdefault}{bx}{n}

\def\sK{{\mathbb{K}}}
\def\sL{{\mathbb{L}}}

\def\sN{{\mathbb{N}}}

\def\sU{{\mathbb{U}}}
\def\sV{{\mathbb{V}}}

\def\sX{{\mathbb{X}}}

\newcommand{\R}{\mathbb{R}}

\date{}

\title{Arbitrary-Depth Universal Approximation Theorems for Operator Neural Networks}

\author{Annan Yu\thanks{Center for Applied Mathematics, Cornell University, Ithaca, NY 14853. (\email{ay262@cornell.edu})} \and Chlo\'{e} Becquey\thanks{University of Connecticut (\email{chloe.becquey@uconn.edu})} \and Diana Halikias\thanks{Mathematics Department, Cornell University, Ithaca, NY 14853 (\email{dh736@cornell.edu})} \and Matthew Esmaili\\ Mallory\thanks{University of California, Los Angeles (\email{matthewmallory@g.ucla.edu})} \and Alex Townsend\thanks{Department of Mathematics, Cornell University, Ithaca, NY  14853. (\email{townsend@cornell.edu})}}

\begin{document}

\maketitle

\begin{abstract}
The standard Universal Approximation Theorem for operator neural networks (NNs) holds for arbitrary width and bounded depth. Here, we prove that operator NNs of bounded width and arbitrary depth are universal approximators for continuous nonlinear operators. In our main result, we prove that for non-polynomial activation functions that are continuously differentiable at a point with a nonzero derivative, one can construct an operator NN of width five, whose inputs are real numbers with finite decimal representations, that is arbitrarily close to any given continuous nonlinear operator. We derive an analogous result for non-affine polynomial activation functions. We also show that depth has theoretical advantages by constructing operator ReLU NNs of depth $2k^3+8$ and constant width that cannot be well-approximated by any operator ReLU NN of depth $k$, unless its width is exponential in $k$. 
\end{abstract}

\section{Introduction}
In the approximation theory of neural networks (NNs), universal approximation theorems (UATs) are statements that establish the density of a class of NNs within a space of mappings. Thus, UATs imply that NNs represent a wide variety of mappings when given appropriate weights and biases. A NN is characterized by its activation function (e.g., ReLU, sigmoid), connectivity (e.g., feedforward, recurrent), width (number of neurons per layer), and depth (number of layers). Operator NNs are a family of NNs for approximating nonlinear operators~\cite{chen,kovachki2021,lu2021learning}. These are critical for learning dynamical systems using DeepONets~\cite{cai2021deepm,lanthaler2021,lu2019deeponet}, inverse mapping problems~\cite{adler2017solving}, and functional data analysis~\cite{rossi2005representation}.  UATs for operator NNs are a fundamental theoretical underpinning for such applications. While there are UATs for wide, shallow operator NNs~\cite{chen}, we derive the first set of UATs for their deep, narrow counterparts. These results are key to understanding the expressibility of deep operator NNs.

There are well-established theoretical advantages of deep, narrow NNs over wide, shallow ones in terms of expressibility. In particular, there are $3$-layer NNs representing radial functions on $\mathbb{R} ^d$ that cannot be approximated by a 2-layer NN to more than a constant accuracy, unless its width is exponential in $d$, where a NN's depth is the number of hidden layers plus one output layer~\cite{eldan2016power}. Moreover, for any $k \in \mathbb Z$, there are $\Theta(k^3)$-deep NNs of constant width which, when restricted to the unit cube $[0, 1]^d$, cannot be approximated by a NN with $\mathcal{O}(k)$ depth, unless it has $\Omega(2^k)$ width~\cite{Telgarsky}. In Section~\ref{approxpowersect}, we construct operators that require exponentially wide ReLU NNs, analogous to~\cite{Telgarsky}. Hence, the improved expressibility of deep, narrow operator NNs over shallow, wide ones is similar to standard NNs. UATs for deep, narrow operator NNs are thus needed to establish their approximation power. 

In Section~\ref{constructsect}, we prove that an operator NN of arbitrary depth and constant width is a universal approximator of nonlinear continuous operators if the activation function is continuously differentiable at a point with nonzero derivative. Our key insight is to use input encoding and reduction of truncated values to decrease the width of the NN to a constant. We thus propagate inputs from one layer to the next with a single neuron. We truncate inputs to a number of digits based on a precision $\varepsilon>0$ and concatenate the truncated values into one value. We extract each truncated input with a decoder function, which we approximate with an arbitrarily deep NN as described in~\cite{kidger}. A related approach is used in~\cite{Shenfixed,Shen3layer}, where inputs are used in their encoded forms. However, we extract the original value from its encoding. 

Our work builds on well-established UATs. An early UAT by~\cite{pinkus} states that an arbitrarily wide one-layer NN with a continuous non-polynomial activation function can approximate all continuous functions on compact sets.
\cite{kidger} prove a UAT for arbitrarily deep NNs with $n$ inputs, $m$ outputs, and of width $n + m + 2$.
An arbitrarily wide operator NNs UAT is given in~\cite{chen}. They prove the standard UAT with a fixed set of weights and biases, then give an arbitrary-width UAT for nonlinear continuous functionals and operators. Though~\cite{kidger} turns the arbitrary-width UAT into one of arbitrary depth, their technique does not extend to operator NNs. In particular, it requires the width of the NN to depend on the size of the sampling device, which depends on the precision $\varepsilon$. Consequently, we would have $n + m(\varepsilon) + 5$ neurons in every hidden layer of the operator NN. This is impractical, since in most cases $m(\varepsilon) \to \infty$ as $\varepsilon \to 0$, resulting in a NN that is both deep and wide. In contrast, our result only requires a constant width operator NN.   

The above UATs all utilize the multi-layer feedforward perceptron (MLP) model. Given an input vector $\mathbf x \in \R^{k_0}$ and activation function $\sigma: \R\to\R$, the output $\varphi(\mathbf x) \in \R^{k_{N}}$ is calculated as
\begin{align*}
    \varphi(\mathbf x) = \mathbf W_N\sigma\left(\mathbf W_{N-1}\left(\cdots\left(\mathbf W_2\sigma\left(\mathbf W_1\mathbf x + \bm\theta_1\right) + \bm\theta_2\right)+\cdots\right) + \bm\theta_{N-1}\right) + \bm\theta_N,
\end{align*}
where $k_0, \hdots, k_{N} \in \mathbb N$, with weights $\mathbf W_{i} \in \R^{k_{i}\times k_{i-1}}$ and biases $\bm\theta_i \in \R^{k_i}$. Here, 
$\sigma$ is applied entry-wise to the vector, i.e., $\sigma(\mathbf a)_j = \sigma( a_{j})$. UATs concern the density of the following space: 
\[
\mathcal M(\sigma) := \text{span}\{\sigma(\mathbf w^\top \mathbf x - \theta)\mid \theta \in \R, \mathbf w \in \R^n\},
\]
where $n$ is the input dimension. We say $\sigma$ has the \textit{density property} if $\mathcal{M}(\sigma)$ is dense in $\mathcal C(\R^n)$ equipped with the topology of uniform convergence on compact sets. The definition of the density property is independent of the dimension $n$ of the input space. Moreover, all continuous, non-polynomial activation functions have the density property~\cite{pinkus}.

\paragraph{Main Contributions.}
We show that deep, narrow NNs are better than shallow, wide ones at approximating certain continuous nonlinear operators, in the sense that significantly fewer neurons are needed to achieve the same accuracy (see Theorem~\ref{telgarsky}). We also show that after truncating inputs, deep NNs of width five can be used to uniformly approximate continuous real-valued functions on compact sets, regardless of the domain's dimension (see Theorems~\ref{thm.truncateUAT} \&~\ref{thm.truncateUAT2}). Finally, we give the first arbitrary-depth UAT for operators with a general class of activation functions (see Theorem~\ref{thm.main}).

\section{Advantages of Depth for Operator Neural Networks}\label{approxpowersect}

There are many advantages of deep, narrow NNs over wide, shallow ones. In particular, some functions are computable by a NN with two hidden layers but require exponentially many neurons of a NN with one hidden layer~\cite{eldan2016power}. This demonstrates the expressive power of deep NNs.  However, this result is achieved by considering the $L^2$ distance between two functions on the entirety of $\R^d$. As our main results are concerned with approximating continuous functions and operators on compact sets, we prove the following more powerful result than~\cite{eldan2016power} for the operator case, inspired by  Theorem~1.1 of~\cite{Telgarsky}.

\begin{theorem}\label{telgarsky}
Let $\mathbb X$ be a Banach space, $\mathbb K_1 \subseteq \mathbb X$ be compact, and $\mathbb V \subseteq \mathcal C(\mathbb K_1)$ be compact. Then, for any integers $n, k \geq 1$, there exists a nonlinear continuous operator $G_k : \mathbb V \to \mathcal C\left([0,1]^n\right)$ such that
\begin{enumerate}
    \item There is a ReLU NN $\varphi : [0, 1]^n \to \R$ of depth $2k^3 + 8$ and width in $\Theta(1)$ such that $\varphi(\mathbf y) = G_k(u)(\mathbf y)$, for any $u \in \mathbb V$ and $\mathbf y \in [0,1]^n$.
    \item Let $m \geq 1$ be an integer. Let $\psi: [0,1]^{n+m} \to \R$ be a ReLU NN with $n + m$ inputs, depth $\leq k$, and $\leq 2^k$ total nodes. Then for any prescribed $x_1, \hdots, x_m \in \mathbb K_1$ and $u \in \mathbb V$, we have 
    \begin{align*}
        \int_{[0,1]^d}\left|G_k(u)(\mathbf y) - \psi\left(u(x_1), \hdots, u(x_m), \mathbf{y}\right)\right|\,d\mathbf y \geq \frac{1}{64}.
    \end{align*}
\end{enumerate}
\end{theorem}

\begin{proof}
Let $k\geq1$ and $\varphi: [0,1]^n \to \R$ be the ReLU NN constructed in  Theorem~1.1 of~\cite{Telgarsky} with depth $2k^3 + 8$ and width in $\Theta(1)$. The first statement of our theorem follows from considering the constant operator $G_k: \mathbb V \to \mathcal C\left([0,1]^n\right)$, $u \mapsto \varphi$. To prove 2, let $\psi: [0,1]^{n+m}\to\R$ be any ReLU NN of depth $\leq k$ with $\leq 2^k$ total nodes. Let $x_1, \hdots, x_m \in \mathbb K_1$ be the prescribed sampling device and $u \in \mathbb V$. Then, define $\psi_u$ as follows:
\begin{align*}
    \psi_u : [0,1]^n \to \R,\ \ \ \mathbf y \mapsto \psi\left(u(x_1), \hdots, u(x_m), \mathbf y\right).
\end{align*}
Since $u(x_1), \hdots, u(x_m)$ can be added onto the first layer's bias term, $\psi_u$ is a NN with $n$ inputs, $\leq k$ layers, and $\leq 2^k$ total nodes. The second statement of the theorem holds by  Theorem~1.1 of~\cite{Telgarsky}, as the ReLU activation function is a $(1,1,1)$-semi-algebraic gate.
\end{proof}

Theorem~\ref{telgarsky} illustrates that increasing the depth of a NN can make operator approximation much less expensive. This suggests that UATs for deep operator NNs comprise an important contribution to our understanding of the limitations of deep learning and expressibility of nonlinear operators.

\section{Construction of the Deep Narrow Operator Neural Network}\label{constructsect}

We present two results on the existence of a deep NN approximation of a nonlinear continuous operator. One is an explicit reconnection of an existing wide NN and the other is an abstract existence argument. In this section, $\sX$ is a Banach space, and $\sK_1 \subset \sX$ is compact. Let $\sV \subset \mathcal{C}(\sK_1) := \mathcal{C}(\sK_1, \R)$ be compact in $\mathcal{C}(\sK_1)$, which is equipped with the topology induced by the uniform norm. Suppose that $n \in \mathbb{N}$, $\sK_2 \subset \R^n$ is compact, and $G: \mathbb V \rightarrow \mathcal{C}(\sK_2)$ is a nonlinear continuous operator. In~\cite{chen}, it is shown that $G$ can be uniformly approximated by a 4-layer NN if the activation function has the density property. More precisely, given any $\varepsilon > 0$, there are positive integers $M, N, m \in \sN$, real numbers $c_i^k, \zeta_k, \xi_{ij}^k \in \R$, vectors $\boldsymbol\omega_k \in \R^n$, and sensors $x_j \in \sK_1$ such that 
\[
\abs{G(u)(\mathbf{y}) - \sum_{k=1}^N\hspace{-1pt}\left[\sum_{i=1}^M c_i^k \sigma \!\!\left( \sum_{j=1}^m\xi_{ij}^ku(x_j) + \theta_i^k \right)\right]\hspace{-3pt}\sigma(\boldsymbol\omega_k \cdot \mathbf{y} + \zeta_k)} < \varepsilon,
\]
for all $u \in \sV$ and $\mathbf{y} \in \sK_2$. The architecture of this NN is shown in Figure~\ref{fig:chen_chen} (left). The input layer consists of $\mathbf{y} = (y_1, \ldots, y_n)$ and $(u_1, \ldots, u_m) = (u(x_1), \ldots, u(x_m))$. The second layer computes $p_{i}^{k} = \sigma \!\!\left( \sum_{j=1}^m\xi_{ij}^ku(x_j) + \theta_i^k \right)$. The third layer computes $r^k = \sigma(\boldsymbol\omega_k \cdot \mathbf{y} + \zeta_k)$ and $q^k = \sum_{i=1}^M c_i^k p_i^k$. The fourth layer consists of multiplication neurons that compute $s^k = r^kq^k$ for $k = 1, \ldots, N$, whose sum is the output of the NN. 

\subsection{Register-Compute Neural Networks}

In a fully connected feedforward NN, connections between non-consecutive layers are not allowed. Such NNs are ``memoryless,'' as a neuron in the $j$\ts{th} layer receives no input other than the output from the $(j-1)$\ts{th} layer. One can introduce memory into a NN by showing that a neuron with a particular activation function, weights, and bias can uniformly approximate the identity function on a compact set \cite{kidger}. We use such neurons to propagate the inputs through the layers of our NN to use them in later computations. This motivates the following definition of the basic model in our construction.

\begin{definition}
Let $p, q \in \mathbb{N}$. A $(p, q)$-\textit{register-compute NN} is a fully connected feedforward NN with $p+q$ neurons in each hidden layer. In each layer, $p$ neurons are called \textit{registers}, ordered so that the only nonzero weight in the $j$\ts{th} register of layer $i$ is from the output of the $j$\ts{th} register of layer $i-1$.
\end{definition}

Although all pairs of neurons in consecutive layers are connected in a fully connected feedforward NN, we effectively ``disconnect'' non-corresponding registers by setting the weights to be zero.

\subsection{Constructing a Deep Operator NN: Reconnecting the Wide Operator NN}

\begin{figure}
    \centering
    \begin{minipage}{.59\textwidth}
    \includegraphics[scale=.60]{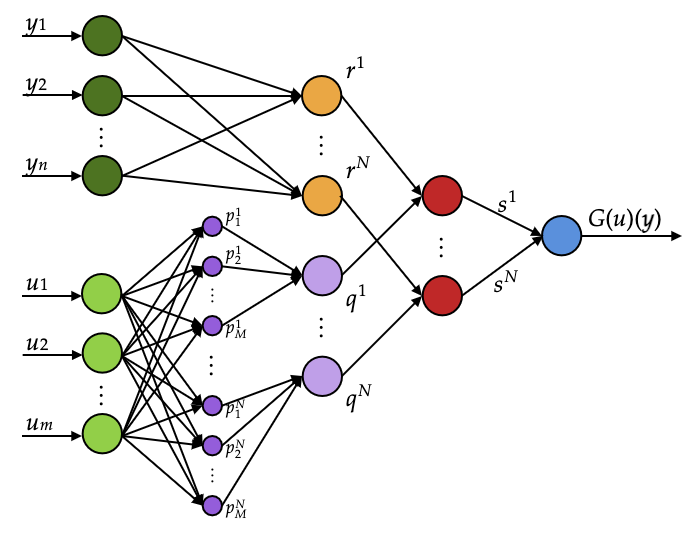}
     \end{minipage}
    \begin{minipage}{.39\textwidth}
    \includegraphics[scale=.55]{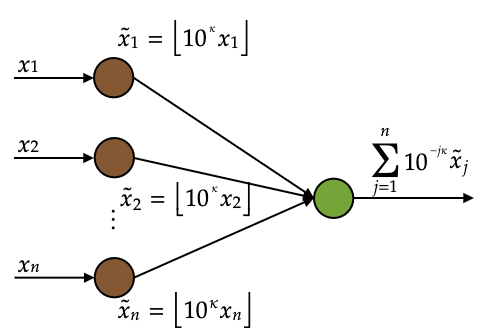}
    \end{minipage}
    \caption{Left: The wide operator NN from~\cite{chen}. Right: Our encoder model.}
    \label{fig:chen_chen}
\end{figure}

We observe that neurons in each hidden layer of the operator NN in~\cite{chen} can be moved one-by-one into different hidden layers. Moreover,   if $\sigma$ has the properties in Theorem~\ref{thm.deepopnet1}, then a $\sigma$-activated neuron can be used to uniformly approximate the identity map $\iota_\sK$ on any compact set $\sK \subset \mathbb R$ (Lemma~4.1 of~\cite{kidger}). This allows us to propagate inputs from one layer to the next. By rearranging the neurons in the shallow, wide operator NN, we get a deep NN whose width depends only on the size of the input layer. 


\begin{theorem}\label{thm.deepopnet1}
Let $\sigma: \R \rightarrow \R$ have the density property. Suppose that $\sigma$ is also continuously differentiable at one or more points with a nonzero derivative. Then, for any $\varepsilon > 0$, there exists a function $F: \R^{m+n} \rightarrow \R$ represented by a $\sigma$-activated NN of width at most $m+n+5$ such that
\[
\abs{G(u)(\mathbf{y}) - F(u(x_1), \ldots, u(x_m), \mathbf{y})} < \varepsilon
\]
for all $u \in \sV$ and $\mathbf{y} \in \sK_2$. Moreover, if a $\sigma$-activated NN of width 3 and depth $L$ approximates the multiplication map $(a, b) \mapsto ab$ on any compact set up to any uniform error, then the network $F$ has depth in $\mathcal{O}((M+L)N)$, where $M, N, m, \{x_j\}_{j=1}^m$ are as in  Theorem~5 of~\cite{chen}.
\end{theorem}

\begin{proof}
Let $H: \R^{m+n} \rightarrow \R$ be the function given by the NN in  Theorem~5 of~\cite{chen} that approximates the operator $G$ to within $\varepsilon / 5$. We construct an $(m+n,5)$-register-compute NN $F$ with $m+n$ inputs, one output, and $(M+L+1)N+1$ layers, where $L$ is a positive integer defined later in~\eqref{eq.L}. Among the 5 neurons that are not registers in each hidden layer, 1 neuron is referred to as the \textit{output augmenter}. 2 neurons are referred to as the \textit{adder} 1 and the \textit{adder} 2, respectively, and the remaining 2 neurons are referred to as the \textit{computation neurons}.

The $m+n$ input layer values are passed into the corresponding $m+n$ registers in the first hidden layer. A register that receives a value $u(x_j)$ is called a $u$-register. If this register is in the $k$th hidden layer, then we denote its output by $u_j^k$. A $y$-register and its output $y_j^k$ are similarly defined. We also define $u_j^0 = u(x_j)$ and $y_j^0 = y_j$. Up to a small error so that $\varepsilon_4$ in~\eqref{eq.3} satisfies $\abs{\varepsilon_4} < \varepsilon / 5$ for all $u \in \sV, \mathbf{y} \in \sK_2$, each register computes a function that is close to the identity function $\iota_\sL$ in $L^\infty(\sL)$, where $\sL$ is the range of the output of the previous register. 

We further divide the $(M+L+1)N$ hidden layers into $N$ sections of $M+L+1$ layers. In the $k$th section, the $i$th adder 1 in each layer computes:
\[
p_i^k = \sigma\!\!\left(\sum_{j=1}^m\xi^k_{ij}u_j^{(M+L+1)(k-1)+i-1}+\theta^k_j\right), \qquad 1\leq i\leq M
\]
using the outputs of the $u$-registers $u_j^{(M+L+1)(k-1)+i-1}$ from the previous hidden layer.

Let $\tilde{q}_i^k = c_i^kp_{i-1}^k + q_{i-1}^k$, where $q_{i-1}^k$ is the output of the $(i-1)$th adder 2 in the $k$th section. We set $p_0^k = q_0^k = 0$. The affine transformation of the $i$th adder 2 in $k$th section computes $\tilde{q}_i^k$ and, together with the activation function, propagates $\tilde{q}_i^k$ using the identity approximation mentioned above, up to a small error so that $\varepsilon_3$ in~\eqref{eq.1} satisfies $\abs{\varepsilon_3} < \varepsilon / 5$ for all $u \in \sV, \mathbf{y} \in \sK_2$. The output is then denoted by $q_i^k$.

The $(M+1)$th adder 1 in the $k$th section computes:
$$r^k = \sigma(\boldsymbol\omega_k \cdot \mathbf{y}^{(M+L+1)(k-1)+M} + \zeta_k)$$
using the outputs of the $y$-registers of the previous hidden layer:
$$\mathbf{y}^{(M+L+1)(k-1)+M} = ({y}^{(M+L+1)(k-1)+M}_1, \ldots, {y}^{(M+L+1)(k-1)+M}_n),$$
The $(M+1)$th adder 2 in the $k$th section propagates $q_M^k$, and its output is denoted by $q^k$.

In the next $L_k$ layers, we can use adder 1, adder 2, and the 2 computation neurons to approximate $r^kq^k$ up to an error so that $\varepsilon_2$ in~\eqref{eq.1} satisfies $\abs{\varepsilon_2} < \varepsilon / 5$ for all $u \in \sV, \mathbf{y} \in \sK_2$ (Proposition~4.9 of \cite{kidger}). This number, denoted by $s^k$, is then added to the output augmenter which, unless otherwise stated, propagates the value from the previous layer, up to an error so that $\varepsilon_1$ in~\eqref{eq.1} satisfies $\abs{\varepsilon_1} < \varepsilon / 5$ for all $u \in \sV, \mathbf{y} \in \sK_2$. The initial value in the output augmenter is set to 0.

We set 
\begin{equation}
    L := \max_{1 \leq k \leq N} L_k. \tag{a}\label{eq.L}
\end{equation}
We assume that any neurons from layer $L_k+1$ to $L$ in the $k$th section do nothing but propagate the values from the previous hidden layer. Once the $N$th section is computed, we add $s^N$ to the augmenter. Now, the value of the augmenter in the $((M+L)N+1)$th layer is given by $S_{u, \mathbf{y}}$, where
\begin{align}
    S_{u, \mathbf{y}} &= \sum_{k=1}^N s^k + \varepsilon_1 = \sum_{k=1}^N q^k r^k  + \varepsilon_1 + \varepsilon_2 = \sum_{k=1}^N \left(\sum_{i=1}^M c_i^kp_{i-1}^k\right) r^k  + \varepsilon_1 + \varepsilon_2 + \varepsilon_3 \label{eq.1} \\ 
    &= \sum_{k=1}^N \left[\sum_{i=1}^M c_i^k \sigma\left(\sum_{j=1}^m\xi^k_{ij}u_j^{\ell(k)+i-1}+\theta^k_j\right)\right] g(\boldsymbol\omega_k \cdot \mathbf{y}^{\ell(k)+M} + \zeta_k) + \varepsilon_1 + \varepsilon_2 + \varepsilon_3 \nonumber \\
    &= \sum_{k=1}^N \left[\sum_{i=1}^M c_i^k \sigma\left(\sum_{j=1}^m\xi^k_{ij}u(x_j)+\theta^k_j\right)\right] g(\boldsymbol\omega_k \cdot \mathbf{y} + \zeta_k) + \varepsilon_1 + \varepsilon_2 + \varepsilon_3 + \varepsilon_4 \label{eq.3} \\ 
    &= H(u(x_1), \ldots, u(x_m), \mathbf{y}) + \varepsilon_1 + \varepsilon_2 + \varepsilon_3 + \varepsilon_4,
\end{align}
where $\ell(k) = (M+L+1)(k-1)$. Since $\abs{\varepsilon_j} < \varepsilon / 5$ for $j = 1, 2, 3, 4$, we have 
\begin{align*}
    \abs{G(u)(\mathbf{y}) - S_{u, \mathbf{y}}}
    &\leq \abs{G(u)(\mathbf{y}) - H(u(x_1), \ldots, u(x_j), \mathbf{y})} \\&\qquad\qquad\qquad+ \abs{(H(u(x_1), \ldots, u(x_j), \mathbf{y}) - S_{u,\mathbf{y}}} \\
    &\leq \frac{\varepsilon}{5} + \abs{\varepsilon_1} + \abs{\varepsilon_2} +\abs{\varepsilon_3} +\abs{\varepsilon_4} < \varepsilon
\end{align*}
for all $u \in \sV, \mathbf{y} \in \sK_2$. The result follows as $S_{u, \mathbf{y}} = F(u(x_1), \ldots, u(x_m), \mathbf{y})$.
\end{proof}

\begin{figure}
    \centering
    \includegraphics[width=\textwidth]{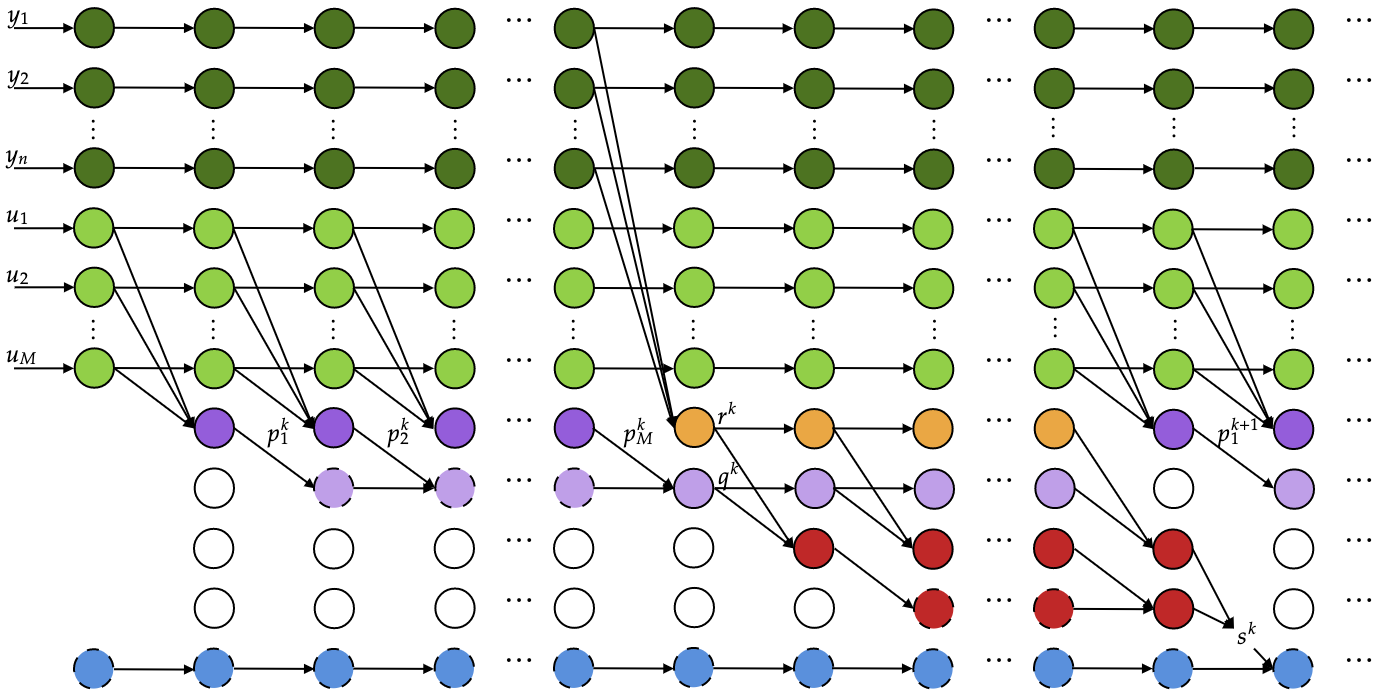}
    \caption{A portion of the deep NN developed by transforming the wide operator NN from~\cite{chen} into a register model. The process can be found in~\cite{kidger}.}
    \label{fig:chenRegister}
\end{figure}
Theorem~\ref{thm.deepopnet1} is a theoretical guarantee that a deep NN can approximate the operator $G$. In particular, the width of our NN does not depend on $M$ and $N$ in  Theorem~5 in~\cite{chen}, where these parameters are obtained abstractly and do not have intuitive interpretations.

Theorem~\ref{thm.deepopnet1} has two shortcomings. First, the total number of neurons in the deep operator NN in Theorem~\ref{thm.deepopnet1} is $\Omega((m+n)(M+L)N)$, whereas that of a shallow, wide NN in~\cite{chen} is $\mathcal{O}(m+n+MN)$. We emphasize, however, that the NN in Theorem~\ref{thm.deepopnet1} is not necessarily the simplest one to achieve an $\varepsilon$-approximation. In fact, deep NNs can outperform the shallow ones in approximating certain operators (see Section~\ref{approxpowersect}).

Second, the NN's width depends on $m$, which in turn depends on $\varepsilon$. Thus, while we have eliminated the dependence of the width on $M$ and $N$, the number of sensors is reflected in the width, and a large sampling device is needed to achieve an accurate approximation, making the NN both arbitrarily deep and arbitrarily wide. To address this, we have two avenues. First, we find a relationship between $m$ and $\varepsilon$. Proving some rate of growth of $m$ with respect to $\varepsilon$ would make Theorem~\ref{thm.deepopnet1} more informative, as in~\cite{lu2019deeponet}, for example. However, results of this type are in a more specific context, and a relationship between $m$ and $\varepsilon$ in the general setting is challenging. Also, to prevent the width of our deep NN from growing too fast as $\varepsilon \rightarrow 0$, we would like to have $m(\varepsilon) = \mathcal{O}(\log(1/\varepsilon))$. So far, we are not aware of any existing result that demonstrates that $m(\varepsilon) = \mathcal{O}(\log(1/\varepsilon))$ is possible. Instead, we need to find a way to reduce the width of the network, regardless of the number of inputs.

\subsection{Input Encoding and Reduction}


It has been shown that arbitrarily-deep NNs of width $m+3$ can uniformly approximate functions $f: \sK \rightarrow \R$, where $\sK \subset \R^m$ is a compact set \cite{kidger}.
We further reduce this width to a constant, eliminating the dependence on $m$. However, it is known that for certain activation functions, the width $m$ is not enough to uniformly approximate continuous functions on compact sets~\cite{hanin2017approximating, lu2017expressive}. Therefore, we slightly modify the architecture of the NNs to make them more flexible.

In~\cite{kidger}, $m$ inputs are propagated throughout the entire network, which requires $m$ neurons in each hidden layer. Our trick is to truncate the inputs using the floor function and then encode them into a single neuron. This single neuron is then propagated using only one neuron from one hidden layer to the next and is decoded when necessary. When decoding, the inputs are decoded one-by-one, and then we immediately pass the decoded value into the computation neurons.

We first define terminology for truncating and encoding inputs. A \textit{truncation neuron} takes an input $x$ and produces the output $\floor{10^\kappa x}$, where $\kappa$ is an arbitrary integer. A \textit{NN with truncated inputs} is a NN where every input has been passed through a truncation neuron. The \textit{width} of the NN with truncated inputs is the size of the largest hidden layer, ignoring the truncation neurons applied immediately to the input layer. Therefore, a $(p, q)$-\textit{register-compute} NN with truncated inputs is a NN with truncated inputs that is a $(p, q)$-register-compute NN if the outputs of the truncation layer are viewed as the inputs of the NN. A $\sigma$-\textit{activated} NN with truncated inputs is a NN whose neurons in all hidden layers have $\sigma$ as the activation function, ignoring the truncation neurons applied immediately to the input layer.

\begin{figure}
    \centering
    \includegraphics[width=\textwidth]{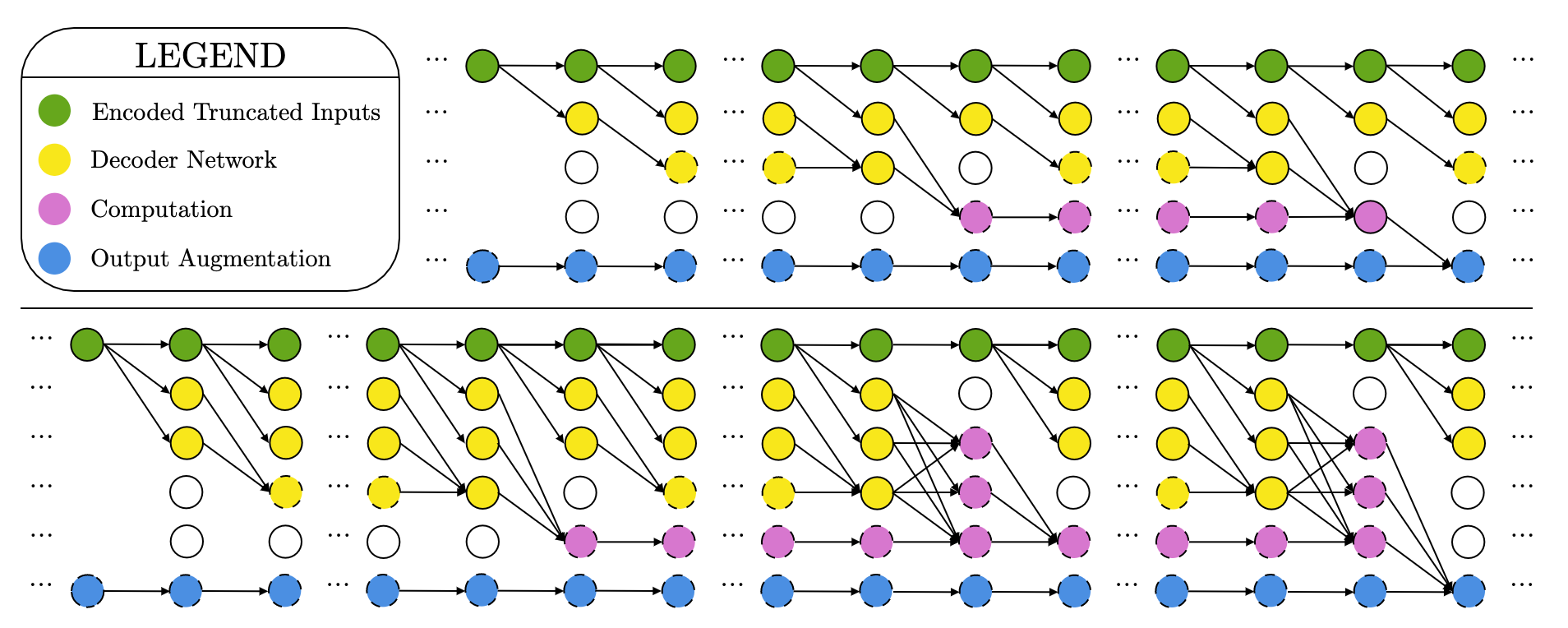}
    \caption{Our decoder model for non-polynomial $\sigma$ (top) and polynomial $\sigma$ (bottom).}
    \label{fig:decoderNonpoly}
\end{figure}

\begin{theorem}\label{thm.truncateUAT}
Suppose $\sigma$ is non-polynomial and continuously differentiable at one or more points with nonzero derivative. Let $\sK \subset \R^n$ be a compact set, and let $f: \sK \rightarrow \R$ be continuous. For each $\varepsilon > 0$, there is a function $g: \R^n \rightarrow \R$ represented by a $\sigma$-activated NN with truncated inputs of width 5 such that 
\[
\abs{f(\mathbf{x}) - g(\mathbf{x})} < \varepsilon, \qquad \mathbf{x} \in \sK.
\]
\end{theorem}
\begin{proof}
Without loss of generality, we let $\sK \subset (0,1)^n$. Otherwise, we scale and translate the domain with the truncation neuron and bias terms in the first hidden layer. There is an identity-activated $(n, 2)$-register-compute NN, $h: \R^n \rightarrow \R$, such that $\abs{h(\mathbf{x}) - f(\mathbf{x})} < \varepsilon/3$ for all $\mathbf{x} \in \sK$~\cite{kidger}. Each of the $n$ inputs is passed into a unique register in the first hidden layer, then propagated by the corresponding register in each hidden layer. Among the two remaining non-register neurons in each hidden layer, one neuron is the \textit{computation neuron}, which applies an affine transformation to the outputs of the registers in the previous hidden layer. The other neuron is the \textit{augmentation neuron}, which sums the outputs of the computation neuron and the augmentation neuron in the previous layer. The output of the first augmentation neuron is set to zero.

Since the NN $h$ is identity-activated, it can be restructured so that each computation neuron only reads one input from the registers and its own output from the previous layer, and applies an affine transformation. To see this, let $\mathbf{x} \mapsto \sum_{j=1}^\ell w_j x_j + b$ be a computation neuron. We replace the layer of this neuron by $\ell$ layers and use the computation neuron from each of the $\ell$ layers to compute $w_1x_1, w_1x_1+w_2x_2, \ldots, \sum_{j=1}^{\ell-1} w_jx_j, \sum_{j=1}^{\ell} w_jx_j + b$, respectively. Each of the remaining neurons in the $\ell$ layers applies the identity to the corresponding output from the previous layer. Let $L+1$ be the depth of this restructured NN.

We now show how we can store input approximations in a single neuron. For large $\kappa \in \mathbb{N}$, we let $\tilde{x}_j = \floor{10^\kappa x_j}$ for $1\leq j\leq n$ be the truncated inputs. The register in the first hidden layer computes 
\[
r := \sum_{j=1}^n 10^{-j\kappa}\tilde{x}_j = 10^{-\kappa }\tilde{x}_1 + 10^{-2\kappa }\tilde{x}_2 + \cdots + 10^{-n\kappa }\tilde{x}_n,
\]
where the remaining registers take the input $r$ and pass it as the output.  Now, we define a series of \textit{decoder functions}, $\varphi_1, \ldots, \varphi_n$. Every $a = 10^{-n\kappa}M \in [0,1), M \in \mathbb{N}_0$ can be expanded uniquely as 
\[
a = a_1 10^{-\kappa} + a_2 10^{-2\kappa} + \cdots + a_m 10^{-n\kappa},
\]
where $a_1, \dots , a_m$ are integers in $[0, 10^\kappa]$. We set $\varphi_j([a-10^{-n\kappa-1}, a+10^{-n\kappa-1}]) := 10^{-n\kappa}a_j$ for each $a$ and then extend $\varphi_j$ to the interval $[0,1]$ continuously by the Tietze Extension Theorem.

Now, we construct a $(1, 4)$-register-compute NN with truncated inputs. We let $p: \R^n \rightarrow \R$ be the function it represents. Unlike most fully connected feedforward NNs, the neurons in each layer have different activation functions. The register uses the identity activation function. Among the four remaining neurons in each hidden layer, one neuron is called the \textit{computation neuron}, and one neuron is called the \textit{augmentation neuron}, which uses the identity activation function. The remaining two neurons are called the \textit{decoder neurons}, which use $\sigma$ as the activation function.

Let $i \in \{1, \ldots, L\}$. In the NN $h$, by assumption, only one of $x_1, \dots, x_n$, say $x_j$, is read by the computation neuron in the $i$th layer.  We construct the NN $p$ by building $L+1$ chunks, where the last chunk is the output layer. To construct the $i$th chunk, we use the two decoder neurons from each hidden layer together with the register to approximate $\varphi_j(r)$ up to a small error, as in Proposition~4.9 of~\cite{kidger}. We note that $\varphi_j(r) = 10^{-\kappa}\tilde{x}_j \approx x_j$. This decoded value is then passed into the computation neuron for the affine transformation done at the $i$th layer in the NN $h$.

Compared to $h$, the difference in the output of $p$ is induced by two steps: the truncating $x$ to obtain $10^{-\kappa}\tilde{x}$, and decoding to obtain an approximation of $\phi_j(r) = 10^{-\kappa}\tilde{x}$. The first error can be made arbitrarily small by taking $\kappa$ large enough and the second error can also be made arbitrarily small as in the previous paragraph. Thus, we can construct the NN $p$ so that $\abs{p(\mathbf{x}) - h(\mathbf{x})} < \varepsilon / 3$ for $\mathbf{x} \in \sK$.

It remains to construct a NN with truncated inputs that only uses $\sigma$ as the activation function. To do so, we define a NN $g: \R^n \rightarrow \R$ whose architecture completely inherits that of $p$, except the registers, computation neurons, and augmentation neurons are $\sigma$-activated. As before, we may use a $\sigma$-activated neuron to mimic the identity activation function. For the register, we make the approximate identity accurate enough so that the perturbed value of $a$, denoted by $\tilde{a}$, always satisfies $\abs{a - \tilde{a}} < 10^{-n\kappa-1}$. Hence, we have that $\varphi_j(\tilde{a}) \equiv \varphi_j(a)$ throughout the entire NN. Since the values in the computation neurons and the augmentation neurons can be propagated arbitrarily accurately, we have $\abs{g(\mathbf{x}) - p(\mathbf{x})}< \varepsilon / 3$ for $\mathbf{x} \in \sK$.  Therefore, we have 
\[
\abs{g(\mathbf{x}) - f(\mathbf{x})} \leq \abs{g(\mathbf{x}) - p(\mathbf{x})} + \abs{p(\mathbf{x}) - h(\mathbf{x})} + \abs{h(\mathbf{x}) - f(\mathbf{x})} < \varepsilon, \qquad \mathbf{x} \in \sK.
\]
\end{proof}
Theorem~\ref{thm.truncateUAT} shows that truncating inputs allows any continuous function on a compact set to be uniformly approximated by deep NNs of constant width. This independence of width and dimension overcomes the problematic growth of the size of the sampling device in~\cite{lu2019deeponet}. Since non-affine polynomial activation functions satisfy the arbitrary-depth UAT, we obtain Theorem~\ref{thm.truncateUAT2}. Analogous to Theorem~\ref{thm.truncateUAT}, which extends Proposition~4.9 of~\cite{kidger}, Theorem~\ref{thm.truncateUAT2} naturally extends Proposition~4.11 of~\cite{kidger}. As opposed to wide NNs, deep NNs with (non-affine) polynomial activation functions approximate continuous functions nicely.

\begin{theorem}\label{thm.truncateUAT2}
Let $\sK \subset \R^n$ be a compact set and $f: \sK \rightarrow \R$ be a continuous function. For each $\varepsilon > 0$, there is a function $g: \R^n \rightarrow \R$ represented by a $\sigma$-activated NN with truncated inputs of width 6 such that 
\[
\abs{f(\mathbf{x}) - g(\mathbf{x})} < \varepsilon, \qquad \mathbf{x} \in \sK.
\]
\end{theorem}

\begin{proof}

Without loss of generality, let $\sK \subset (0,1)^n$. Let $p = \sum_{j=1}^\ell p_j: \sK \rightarrow \R$ be a polynomial with monomials $p_j$ such that 
$\abs{f(\mathbf{x}) - p(\mathbf{x})}< \varepsilon / 3$ for $\mathbf{x} \in \sK$. There is a $(n,4)$-register-compute NN $h: \R^n \rightarrow \R$ satisfying $\abs{h(\mathbf{x}) - p(\mathbf{x})} < \varepsilon / 3$ for $\mathbf{x} \in \sK$, where each hidden layer contains $n$ identity-activated registers that propagate the $n$ inputs, one $\sigma$-activated augmentation neuron that stores the output and never takes any register as an input, and three $\sigma$-activated computation neurons that compute the monomials $p_j$ (Proposition~4.6 and Proposition~4.11 in~\cite{kidger}).

The computation neurons take no more than one value from the registers as the input in each layer. Moreover, when these neurons need inputs from one of the registers, the outputs of all but possibly one in the previous layer do not become the input of any other neurons in the current layer. 

As in the proof of the Theorem~\ref{thm.truncateUAT}, we can construct a $\sigma$-activated NN $g: \R^n \rightarrow \R$ with truncated input of width 8 such that $\abs{h(\mathbf{x}) - g(\mathbf{x})} < \varepsilon / 3$, for $\mathbf{x} \in \sK$. In particular, each hidden layer contains one register that propagates the encoded input as in the proof of Theorem~\ref{thm.truncateUAT}, three decoder neurons that, together with the register, approximate the decoders $\varphi_1, \ldots, \varphi_n$ (Proposition~4.11 in~\cite{kidger}), 3 computations neurons as in $h$, and 1 augmentation neuron as in $h$.

Finally, when a decoder neuron is activated and its output becomes an input to the next layer, the computation neurons read the input from the register. However, when the computation neurons read inputs from the register, the outputs of two of them in the previous layer are not used in the current layer. Thus, two computation neurons can be reused in the architecture of the decoder. Hence, only $3+3-2=4$ neurons are used to implement the decoder and the computation unit, and consequently, $g$ can be realized by a NN with truncated inputs of width $6$. Now, for all $\mathbf{x} \in \sK$, we have
\[
\abs{g(\mathbf{x}) - f(\mathbf{x})} \leq \abs{g(\mathbf{x}) - h(\mathbf{x})} + \abs{h(\mathbf{x}) - p(\mathbf{x})} + \abs{p(\mathbf{x}) - f(\mathbf{x})} < \varepsilon.
\]
\end{proof}

We note that the success of the encoder/decoder does not depend on the representation being decimal. They can be equivalently constructed using binary representations of numbers, so that if the operation in the truncation neuron is $x \mapsto \floor{2^\kappa x}$, Theorem~\ref{thm.truncateUAT} and Theorem~\ref{thm.truncateUAT2} still hold. This result is more relevant to most modern machines, which store floating-point numbers with finitely many bits, conduct floating-point arithmetic in binary, and perform $x \mapsto 2^\kappa x$ easily.

\subsection{Constructing the Second Deep Operator NN: an Abstract Approach}
Now, we have the tools to eliminate the dependence of the NN's width on the size of the sampling device. We adopt an abstract strategy to construct an operator NN with truncation whose width is a constant. To do so, we view the NN in~\cite{chen} as a function $f$ from $\R^{m+n}$ to $\R$, for we encode the input function $u$ as $m$ values. Therefore, to approximate the operator $G$, it suffices to approximate $f$ uniformly.

\begin{theorem}\label{thm.main}
Let $\sigma$ be non-polynomial (resp. non-affine), continuously differentiable at one or more points with nonzero derivative. Then, for every $\varepsilon > 0$, there are points $x_1, \ldots, x_m \in \sK_1$ and a function $F: \mathbb{R}^{m+n} \rightarrow \R$ given by a $\sigma$-activated NN with truncated inputs of width 5 (resp. 6), such that 
\[
\abs{G(u)(\mathbf{y}) - F(u(x_1), \ldots, u(x_m), \mathbf{y})} < \varepsilon
\]
for all $u \in \sV$ and $\mathbf{y} \in \sK_2$. Moreover, $m$ is independent of $\sigma$.
\end{theorem}

\begin{proof}
Define $\sU_j = \{u(x_j) \mid u \in \sV\}$ and $\sU = \prod_{j=1}^m \sU_j$. The evaluation map $\phi_j: \sV \rightarrow \R, u \mapsto u(x_j)$ is continuous. Hence, $\sU_j = \phi_j(\sV)$ is compact for each $j$, and so is $\sU$ by Tychonoff's Theorem. Let $g: \R \rightarrow \R$ be an arbitrary function with the density property.  This function induces points $x_1, \ldots, x_m \in \sK_1$ and a function $H: \sU \times \sK_2 \rightarrow \R$ such that 
\[
\abs{G(u)(\mathbf{y}) - H(u(x_1), \ldots, u(x_m), \mathbf{y})} < \varepsilon / 2
\]
for any $u \in \sV$ and $\mathbf{y} \in \sK_2$ (Theorem~5 in~\cite{chen}). Let $F: \R^{m+n} \rightarrow \R$ be the function represented by the NN with truncated inputs constructed in Theorem~\ref{thm.truncateUAT} or Theorem~\ref{thm.truncateUAT2} associated with the function $H$ and the approximation error $\varepsilon / 2$. The statement of the theorem follows from the triangle inequality and the fact that $g$ is arbitrary, making $m$ independent of $\delta$.
\end{proof}

Compared to Theorem~\ref{thm.deepopnet1}, Theorem~\ref{thm.main} gives us a deep operator NN whose width is constant. Moreover, it allows us to use non-affine polynomial activation functions, which are known to be powerless in approximating using the 2-layer networks~\cite{pinkus}. Inspired by Proposition~4.17 of~\cite{kidger}, we have the following extension of Theorem~\ref{thm.main}.

\begin{corollary}\label{cor.oneless}
Let $\sigma: \R \to \R$ be a polynomial such that $\sigma^\prime(\alpha) = 0$ and $\sigma^{\prime\prime}(\alpha) \neq 0$ for some $\alpha \in \R$. Then, for every $\varepsilon > 0$, there exist points $x_1, \ldots, x_m \in \sK_1$ and a function $F: \R^{m+n} \rightarrow \R$ represented by a $\sigma$-activated NN with truncated inputs of width 5, such that 
\[
\abs{G(u)(\mathbf{y}) - F(u(x_1), \ldots, u(x_m), \mathbf{y})} < \varepsilon, \qquad u \in \sV, \quad \mathbf{y} \in \sK_2.
\]
\end{corollary}

\begin{proof}
The proof follows from Theorem~\ref{thm.truncateUAT2}. The NN $h$ in Theorem~\ref{thm.truncateUAT2} can be implemented using a $(m,3)$-register-compute NN. Two neurons can implement the decoder in $g$. When the decoder is activated, it uses one of the two computation neurons. The rest of the proof is then analogous to the proof of Theorem~\ref{thm.main} and the width of $g$ is $2+2+2-1 = 5$, where the first ``2" corresponds to the register and the output augmenter. The second and the third ``2"s are the number of neurons needed to implement the decoder and the number of computation neurons, respectively. 
\end{proof}


Corollary~\ref{cor.oneless} in combination with the non-polynomial $\sigma$ case means that ``most" activation functions require our NN with truncated inputs to have a width of 5. This is a slight improvement compared to Theorem~\ref{thm.main}, in which we require a width of 6 when $\sigma$ is a non-affine polynomial.

\section{Conclusion}
This paper proves that arbitrary-depth operator NNs with a large class of activation functions are universal approximators. Our main theorem is a UAT for operator NNs of width 5 with a non-polynomial that is continuously differentiable at a point with nonzero derivative (see Theorem~\ref{thm.main}). Our proof technique is robust enough to handle non-affine polynomial activation functions too (see Theorem~\ref{thm.truncateUAT2} and Corollary~\ref{cor.oneless}). We also construct an operator ReLU NN of depth $2k^3+8$ and constant width that cannot be well-approximated by any operator ReLU NN of depth $k$, unless its width is exponential in $k$ (see Theorem~\ref{telgarsky}). This demonstrates that deep, narrow NNs are better than shallow, wide ones at approximating certain continuous nonlinear operators. We hope that this adds theoretical justification to those that use deep operator NNs.

\subsubsection*{Acknowledgments}

We are thankful for the NSF RTG grant no.~1645643 that partially supported this research. We are also grateful for the NSF grant no. DMS-2045646 and the NSF-GRFP grant no. DGE-1650441. 

\bibliography{arXiv_operatorNN}
\bibliographystyle{siam}


\end{document}